\documentclass[runningheads]{llncs}

\usepackage{amsmath,amsfonts,bm}

\def\eqref#1{equation~\ref{#1}}

\def\1{\bm{1}}

\DeclareMathAlphabet{\mathsfit}{\encodingdefault}{\sfdefault}{m}{sl}
\SetMathAlphabet{\mathsfit}{bold}{\encodingdefault}{\sfdefault}{bx}{n}

\DeclareMathOperator*{\argmin}{arg\,min}

\usepackage[hidelinks]{hyperref}

\usepackage{url}

\usepackage[dvipdfmx]{graphicx}

\usepackage[utf8]{inputenc} %
\usepackage[T1]{fontenc}    %
\usepackage{hyperref}[hidelinks]       %
\usepackage{url}            %
\usepackage{booktabs}       %
\usepackage{amsfonts}       %
\usepackage{nicefrac}       %
\usepackage{microtype}      %
\usepackage{xcolor}         %

\usepackage{amsmath}
\usepackage{amssymb}
\usepackage{mathtools}

\usepackage[capitalize,noabbrev]{cleveref}
\crefname{equation}{}{}

\usepackage[textsize=tiny]{todonotes}

\usepackage{colortbl}
\usepackage{subfigure}

\usepackage[noend]{algorithmic}
\usepackage{algorithm}
\renewcommand{\algorithmicrequire}{\textbf{Input:}}
\renewcommand{\algorithmicensure}{\textbf{Output:}}

\usepackage{xcolor}

\usepackage{comment}
\newcommand{\del}[1]{}
\newcommand{\indep}{\perp\!\!\!\!\perp}
\DeclareRobustCommand{\citet}[1]{\cite{#1}}
\newcommand{\citep}[1]{\cite{#1}}
\newcommand{\citeapp}[1]{} %
\newcommand{\evid}{E}  %

\renewcommand{\scriptsize}{\small}

\def\tablescale{0.8}

\titlerunning{A Distributionally-Robust Framework for Causal Estimation} %

\begin{document}

\title{A Distributionally Robust Framework for Nuisance in Causal Effect Estimation}

\author{Akira Tanimoto\inst{1}\orcidID{0000-0003-0459-3993}}

\institute{NEC Corporation}

\maketitle

\begin{abstract}
Causal inference requires evaluating models on balanced distributions between treatment and control groups, while training data often exhibits imbalance due to historical decision-making policies.
Most conventional statistical methods address this distribution shift through inverse probability weighting (IPW), which requires estimating propensity scores as an intermediate step.
These methods face two key challenges: inaccurate propensity estimation and instability from extreme weights.
We decompose the generalization error to isolate these issues—propensity ambiguity and statistical instability—and address them through an adversarial loss function.
Our approach combines distributionally robust optimization for handling propensity uncertainty with weight regularization based on weighted Rademacher complexity.
Experiments on synthetic and real-world datasets demonstrate consistent improvements over existing methods.

\end{abstract}

\section{Introduction}
\label{sec: intro}
Causal inference enables us to assess the impact of a treatment action.
Its application originated in the policy-making field~\cite{lalonde1986evaluating} including healthcare~\cite{sanchez2022causal}.
Recently, the focus was expanded to individualized decision-making such as precision medicine~\cite{sanchez2022causal}, recommendation~\cite{schnabel2016recommendations,bonner2018causal}, and advertisement~\cite{sun2015causal,Wang2015RobustTC} with the help of advanced machine learning-based methods.

We estimate the causal effect of a treatment action (e.g., prescription of a specific medicine).
That is, we need accurate predictions of both potential outcomes with and without the treatment to take its difference since the supervision of the actual effect itself is never given.
In observational data, actions are not assigned randomly but selected by past decision-makers.
Therefore, models must generalize beyond the factual data distribution to handle the systematically missing counterfactual outcomes.
This is called the fundamental problem of causal inference~\cite{shalit2017estimating}.

Conventional statistical methods for causal inference deal with this action selection bias by matching extraction or importance sampling~\cite{rosenbaum1983central}. %
A representative and versatile approach is inverse probability weighting using propensity scores (IPW)~\cite{austin2011introduction}.
IPW estimates propensity scores—the probability of past treatment decisions—then weights instances by their inverse to account for selection bias.
This two-step strategy has fundamental limitations: First, the final accuracy depends entirely on propensity score precision. Second, extreme weights can concentrate on few samples, reducing effective sample size and causing high estimation variance~\cite{kang2007demystifying}.
The overall accuracy can only be guaranteed asymptotically,
 limiting its applicability to modern non-asymptotic scenarios
such as high dimensional models as DNNs for capturing complex heterogeneity or complex action spaces.

Various countermeasures have been tried to alleviate this 
problem, such as doubly robust (DR) methods~\cite{kang2007demystifying,kennedy2020towards,dudik2014doubly} and double machine learning~\cite{chernozhukov2018double,nie2021quasi}, which are clever combinations of outcome prediction models and only weighting its residuals using estimated propensity.
Nevertheless, the IPW-based approach's limitation is the instability of the two-step procedure in non-asymptotic scenarios.
Its large estimation variance is problematic in cases of high dimensional covariates or cases where the propensity score is close to $0$ or $1$ and thus the effective sample size is limited~\cite{athey2018approximate}.

As in various other fields, advancement with deep neural networks (DNN) has gained substantial attention in causal inference literature~\cite{li2022survey}.
One of the notable advancements made when~\citet{shalit2017estimating} and~\citet{johansson2016learning} introduced DNN to causal inference was representation-level balancing through distribution discrepancy measures.
The representation extracted from the input covariates is encouraged to be {\it balanced}, i.e., independent of the action, by measuring and minimizing the discrepancy between the representation distributions conditioned on the action.
Representation balancing provides non-asymptotic performance guarantees through generalization error bounds~\cite{shalit2017estimating}.
This results in an end-to-end training procedure, free from the concerns in the intermediate estimation problem.
However, it has been pointed out that the guarantee of generalization error of representation balancing requires the unrealistic assumption of invertible representation extraction~\cite{johansson2019support,zhao2019learning,wu2020representation}.
In recent years, researchers have explored approaches that combine weighting with representation learning~\cite{Hassanpour2020Learning,wu2022learning,cheng2022learning}. However, to the best of our knowledge, this line of research has not yet yielded a theoretically grounded approach based on generalization error analysis.

We analyze the generalization error of the weighted loss function using a propensity model, which is implemented as adversarial end-to-end learning of the target and weighting models.
We define the worst-case loss with respect to the nuisance model ambiguity as a general framework.
Our adversarial loss simultaneously accounts for the true propensity's ambiguity and the statistical instability due to skewed weights.
We demonstrate our framework's effectiveness by applying it to two established methods: the doubly robust learner~\cite{kennedy2020towards} and representation-based architectures~\cite{curth2021nonparametric}.
Experiments on both synthetic and real-world datasets show consistent improvements over baseline methods.

\section{Problem setting}
\label{sec:problem}
We consider a standard causal inference framework.
We have observational data $D=\{(x^{(n)}, a^{(n)}, y^{(n)})\}_1^N$ with $N$ i.i.d. instances.
Each instance contains a $d$-dimensional background feature $x^{(n)} \in \mathcal X \subset \mathbb R^d$, a treatment action $a^{(n)} \in \mathcal A = \{0,1\}$, and an outcome $y^{(n)} \in \mathcal Y$.
In the Neyman-Rubin potential outcome framework~\cite{rubin2005causal}, the potential outcomes of both factual and counterfactual actions are expressed as random variables $\{Y_a\}_{a\in \{0,1\}}$, of which only the factual outcome is observed $\left( y^{(n)}=y_{a^{(n)}} \right)$ and the counterfactual one $y_{1-a^{(n)}}$ is missing.

Our goal is to learn a potential outcome function $f: \mathcal X \times \mathcal A \to \mathcal Y$ to estimate the causal effect $\hat \tau(x) := \hat f(x,a=1) - \hat f(x, a=0)$ under the given background feature $x$, or to learn $\hat \tau$ directly.
The estimated effect $\hat \tau(x)$ is expected to approximate the true individualized causal effect defined as the conditional average treatment effect (CATE).
    \begin{align}
        \label{eq: CATE}
        \tau(x) = \mathbb E\left[Y_1 - Y_0 |x \right]
    \end{align}
A typical metric for the estimation accuracy is the MSE of $\tau(x)$, also known as the precision in estimating heterogeneous effects (PEHE) $\mathbb E_x \left[ \left(\tau(x) - \hat \tau(x)\right)^2 \right]$.

As a sufficient condition for consistent learnability of the CATE, we follow the standard set of assumptions in the potential outcome framework~\cite{imbens2015causal}.
\begin{itemize}
    \item $Y^{(n)} \indep A^{(n')} ~~ \forall n\neq n'$ (Stable Unit Treatment Value Assumption)
    \item $(Y_0, Y_1)\indep A \mid X$ (unconfoundedness) 
    \item $0<\mu(a | x)<1 ~~ \forall x,a$ (overlap) 
  \end{itemize}

\section{Related work}
\label{related}

\paragraph{Inverse probability weighting with propensity scores (IPW) and its extension}
IPW is a well-known and universal approach to various causal inference problems.
It balances the distribution by weighting instances with the inverse of their estimated propensity scores.
Taking their inverse increases the estimation variance when the propensity score estimates are extreme.

Orthogonal statistical learning, such as the DR-Learner~\cite{kennedy2020towards} and the R-Learner~\cite{nie2021quasi}, utilize two kinds of nuisance models of predicting outcomes and predicting treatments.
These methods have been shown to be robust to estimation errors for the first-stage nuisance parameters of propensity and outcome models.
Specifically, the errors do not affect the final estimation in the first-order sense of the Taylor expansion.
However, their main limitation lies in high estimation variance when applied to non-asymptotic situations~\cite{athey2018approximate}.

Aiming at robustness for complex DNN-based models, we therefore develop a unified framework that is based on the orthogonal method but also cares about extreme weights.

\paragraph{Representation-based method using neural networks}
Starting with~\cite{johansson2016learning,shalit2017estimating}, a number of causal inference methods based on DNNs and representation balancing have been proposed~\cite{li2022survey}.
The representation-based methods have been demonstrated to be superior in complex problems such as nonlinear responses~\cite{johansson2016learning}, large actions spaces~\cite{tanimoto2021regret} including continuous~\cite{lopez2020cost} or structured spaces~\cite{harada2021graphite}, and so forth.
These are end-to-end methods based on adversarial formulations.
They virtually evaluate the worst-case with respect to the uncertainty of the model by distribution discrepancy between the representations of covariates in treated and control groups.
On the other hand, representation balancing has certain limitations in an unrealistic theoretical assumption that the representation extractor should be invertible.
It is shown that unobservable error terms arise when the invertibility is violated~\cite{johansson2019support,zhao2019learning,wu2020representation}.

A solution to this problem is the representation decomposition~\cite{Hassanpour2020Learning,wu2022learning,cheng2022learning}. 
They aim to identify confounding factors that affect both action selection and the outcomes and weights with only those factors.
Joint optimization approaches have also been proposed for ATE estimation~\cite{shi2019adapting}.
However, their formulations have no guarantee as a joint optimization, e.g., in the form of generalization error bound for all possible combinations of the target and nuisance models.
Even though estimation of weights with other model parameters (i.e., outcome model and representations) fixed is guaranteed, and vice versa, though joint optimization may lead to unexpected results.
Weights should be optimized to balance the distributions, but especially when the noise is heterogeneous, lowering the weights to the noisier regions may help the (wrongly) weighted loss of the outcome prediction model.
Thus, we aim at a principled and versatile weighting method while incorporating the advantages of end-to-end modeling by adversarial formulation.

\paragraph{Distributionally robust optimization}

Our proposed method can be interpreted in the context of distributionally robust optimization (DRO).
DRO aims to achieve robust learning against discrepancies between the empirical distribution of training data and the actual test distribution~\cite{rahimian2019distributionally}.
DRO is formalized as follows.
\begin{align}
    \mathfrak R_\mu(\ell \circ \Theta) := \frac{1}{N} \underset{D, \sigma \sim \{\pm 1\}^N}{\mathbb{E}}\left[\sup _{\theta \in \Theta} \sum_{i=1}^N \sigma_n w^{(n)}_\mu \tau_\theta \left(x^{(n)}\right)\right].
\end{align}
Typically, the ambiguity set $\mathcal U$ is defined by small perturbations to the empirical distribution.
In contrast, we address causal inference where training and test distributions can differ significantly, which difference is roughly estimated by the propensity.
We further address the ambiguity of the propensity estimation and statistical instability due to extreme weights, which is realized by design of $\mathcal U$ in DRO.

\paragraph{Pessimism in offline reinforcement learning}
Recent efforts in offline reinforcement learning revealed the benefit of pessimism on the candidate assessment~\cite{rashidinejad2021bridging,buckman2021importance}.
In reinforcement learning, we consider modeling the cumulative expected reward in the long run as the Q function for assessing each action at each time step.
The Q function is supposed to be maximized with respect to action $a$ during the inference phase.
If the estimation error on an action is optimistic, i.e., if the Q value is overestimated, the action is likely to be selected over other better candidates.
Therefore, conservative modeling of Q-function is preferable~\cite{kumar2020conservative}, i.e., training a model to estimate below the true value when uncertain.
The provable benefit of pessimism has been revealed in recent years~\cite{buckman2021importance}.
We apply this pessimism principle to weighted estimation in causal inference; that is, our method pessimistically estimates the weighted loss with respect to the uncertainty of the weights.
Our method applies the principle of pessimism in training that minimize the balanced loss with loss uncertainty due to propensity ambiguity and extreme weights, which can be evaluated by the generalization error analysis.

\section{Nuisance-Robust Weighting Network}
\label{sec: method}
This section presents our general framework for making plug-in loss functions robust to nuisance uncertainty. 
We first establish the theoretical foundation of our adversarial reformulation. We then implement it through constrained optimization using regularization techniques.
To demonstrate the generality of our approach, we apply it to two distinct architectures: the doubly-robust network (DRNet) for weighted estimation and the shared representation network (SNet) for representation-based learning.

\subsection{Adversarial Reformulation of Plug-in Nuisance}
\label{subsec: general loss}

Most weighting-based approaches are formalized as the following two-step procedure.
That is, 1) estimate the nuisance propensity model $\hat \mu$ with its empirical evidence $\hat E$ (e.g., the likelihood)
\begin{align}
    \label{eq: nuisance estimation}
\hat \mu = \argmin_{\mu \in M} \hat \evid (\mu) 
    \end{align}
and 2) plug it into the target empirical risk $\hat L$ (e.g., an MSE of the CATE estimator $\hat \tau_\theta$)
\begin{align}
\label{eq: plug-in estimation}
    \hat \theta &= \argmin_{\theta \in \Theta} \hat L(\theta; \hat \mu).
\end{align}
The assumed model classes are denoted by $M$ and $\Theta$.
A typical form of \cref{eq: plug-in estimation} is instance weighted loss:
\begin{align}
    \label{eq: weighted loss}
    \hat L(\theta; \mu) = \frac{1}{N} \sum_n^N w^{(n)}_\mu \ell^{(n)}(\theta),
\end{align}
where $w_\mu$ is the instance weight defined by $\mu$ and $\ell$ is the instance-wise loss without weights.
Although we do not necessarily assume this product form of $w$ and $\ell$, we assume that Lipschitz continuity and the upper bound of (weighted) instance-wise loss are the product of the constant and the weight.

Plug-in estimators typically analyze generalization error through stage-wise convergence rates~\cite{oprescu2019orthogonal}.
For complex model classes, both convergence rates and their coefficients matter.
We therefore adopt an end-to-end approach analyzing generalization error directly~\citet{shalit2017estimating}, avoiding two-step estimation issues.

Let the true propensity function $\mu^0(x) = \mathbb E[a|x]$ and its best approximation $\mu^* := \argmin_{\mu \in M} \evid(\mu)$ where $\evid$ denotes the expected loss function for the nuisance.
We assume that $\evid$ is a proper loss, i.e., $\mu^*=\mu^0$ when $\mu^0 \in M$.
Then, the generalization error can be decomposed as follows.
\begin{align}
        \label{eq: generalization error decomposition}
        L(\theta;\mu^0) =& \underbrace{L(\theta;\mu^0) - L(\theta;\mu^*)}_{(a)} \notag\\
        &+\underbrace{L(\theta;\mu^*) - \max_{\mu \in M} \left\{ \hat L(\theta; \mu)\right\}}_\text{(b)} + \max_{\mu \in M} \left\{ \hat L(\theta; \mu)\right\}
    \end{align}
The term (a) is the misspecification error of the nuisance class $M$, which would be zero if the true propensity $\mu^0$ is in $M$.
The term (b) represents the optimistic-side error of the loss $\max_{\mu \in M} \left\{ \hat L(\theta; \mu)\right\}$.
Our strategy is to define a learning objective that balances these terms.
That is, we define the last term as our objective
\begin{align}
    \label{eq: constrained robust loss}
    \hat J(\theta) = \max_{\mu \in M} \hat L(\theta; \mu)
\end{align}
and control the tradeoffs by $M$.

When taking $M$ too large, the misspecification (a) is avoided but the gap between the objective $\hat J(\theta)$ and the generalization error $L(\theta; \mu^0)$ would be large.
To control this tradeoff, we define the ambiguity set of the propensity
\begin{align}
    \label{eq: def ambiguity set}
    \mathcal U = \left\{\mu \in M_0 \;\middle|\; \hat \evid(\mu) \le c \right\},
\end{align}
where $c$ is the tolerance hyperparameter.
A reasonable choice of the tolerance would be by the validation error of the propensity, e.g., $\mathcal U = \left\{\mu \;\middle|\; \hat \evid(\mu) \le \hat \evid(\hat \mu_\mathrm{es}) \right\},$ where $\hat \mu_\mathrm{es}$ is the early stopping solution.
Selecting the base class $M_0$ as a flexible one such as neural networks and restricting it reasonably reduces the gap between $L(\theta; \mu^0)$ and $\hat J(\theta)$ while containing the true propensity in the class.

The second term (b) in \cref{eq: generalization error decomposition} is upper-bounded by the excess risk of the weighted loss since $\mu^* \in M$.
\begin{align}
    \text{(b)}
    \le& L(\theta; \mu^*) - \hat L(\theta; \mu^*)
\end{align}
We can establish a high-probability bound for this term by introducing a weighted variant of the Rademacher complexity with a weighting function $\mu$: %
\begin{align}
    \label{def: weighted Rademacher complexity}
    &\mathfrak R_\mu(\ell \circ \Theta) := \frac{1}{N} \underset{D, \sigma \sim \{\pm 1\}^N}{\mathbb{E}}\left[\sup _{\theta \in \Theta} \sum_{i=1}^N \sigma_n w^{(n)}_\mu \tau_\theta \left(x^{(n)}\right)\right].
\end{align}

Then, we have the following upper-bound.
\begin{theorem}
    \label{thm: weighted excess risk bound with weighted Rademacher}
    Suppose that the instance-wise loss is bounded by $c'$ as $w^{(n)}_\mu \ell^{(n)}(\theta)\le c'$. Then, for any $\delta>0$, with probability at least $1-\delta$ over the choice of a sample $D$, the following holds for all $\theta \in \Theta$.
    \begin{align}
        L(\theta; \mu) - \hat L(\theta; \mu) \le 2 \mathfrak R_\mu(\ell \circ \Theta) + \frac{c'}{2} \sqrt{\frac{\log\left(1/\delta\right)}{N}}.
    \end{align}
\end{theorem}

The weighted Rademacher complexity $\mathfrak R_\mu(\ell \circ \Theta)$ depends on the class $\Theta$.
For a bounded linear class, which serves as a typical example, an upper bound can be established as follows.

\begin{theorem}
    \label{thm: weighted Rademacher}
    Let $\Theta$ be a bounded linear function class, i.e., $f_\theta(x) = \theta^\top x$ with $\|\theta\|\le B$.
    Furthermore, assume that $\|x\|_2 \le X$ for all $x \in \mathcal X$.
    Then, the following holds:
    \begin{align}
        \mathfrak R_\mu(\ell \circ \Theta) \le \frac{BX}{N} \sqrt{\mathbb E_D \|w_\mu\|_2^2}.
    \end{align}
\end{theorem}
The proofs for these theorems can be found in \cref{sec: appendix proof}.

Consistent with findings and proposals in several prior studies~\cite{awasthi2024best,tanimoto2022improving,swaminathan2015counterfactual}, our analysis also demonstrates that the mean squared weight is related to the stability of the loss (b) in \cref{eq: generalization error decomposition}.
In other words, assigning too large weights to a small fraction of the sample compromises the effective sample size.
Therefore, to mitigate the impact of (b), we introduce a regularization on $\mu$ in the form of the squared weights, albeit at the potential cost of increasing the misspecification error (a) in \cref{eq: generalization error decomposition}.
More precisely, we define the class for restricted squared weights $\mathcal R$ as follows.
\begin{align}
    \mathcal R = \left\{ \mu \in M_0 \middle| \frac{1}{N} \sum_n^N \left(w^{(n)}_\mu\right)^2 \le C \right\},
\end{align}
where $C>0$ is a hyperparameter.
Then, we define the class for the nuisance as $M = \mathcal U \cap \mathcal R$.

\subsection{Nuisance-robust Doubly-robust Network (NuDRNet)}
\label{subsec: baseline method}
Under the aforementioned approach, next, we will discuss the application to specific estimation methods.
Among two-step methods with weighting,
we take the doubly-robust learner or DRNet~\cite{kennedy2020towards} as a simple but clever baseline method.
DRNet regresses a transformed target variable on $x$, which is calculated by the plug-in CATE estimate $\hat f_1(x^{(n)}) - \hat f_0(x^{(n)})$ with residual adjusted with weights:
\begin{align}
    \label{eq: transformed outcome DR}
    z^{(n)}_{\hat \mu} = &\hat f_1(x^{(n)}) - \hat f_0(x^{(n)}) +\frac{y_1^{(n)}-\hat f_1(x^{(n)})}{\hat \mu(x^{(n)})}a^{(n)} \\
    &-\frac{y_0^{(n)}-\hat f_0(x^{(n)})}{1-\hat \mu(x^{(n)})}(1-a^{(n)}).
\end{align}
This transformed target $z$ approximates CATE $\tau(x)$ in expectation when either the outcome models $(\hat f_1, \hat f_0)$ or the weighting model $\hat \mu$ is accurate.

We propose nuisance-robust DRNet (NuDRNet) as an extension of DRNet.
Our approach starts with a pre-trained propensity $\hat \mu$ as the initial solution $\mu_0$ and a randomly initialized parameter $\theta$.
We then perform adversarial optimization as described in \cref{eq: constrained robust loss}, using the mean squared error: $\hat L(\theta; \mu) = \frac{1}{N}\sum_{n}\left(z^{(n)}_\mu - \tau_\theta(x^{(n)})\right)^2.$
Unlike DRNet, which uses a fixed pre-trained $\hat \mu$, NuDRNet perturbs $\hat \mu$ adversarially during training of $\hat \tau_\theta$.

Next, we discuss how to incorporate the constraint $\mu \in M$ into a gradient-based update.
A typical implementation is as regularization term such as $\alpha \max\{0, \hat \evid(\mu) -  c\},$
where $c$ is the tolerance parameter and is set to the evidence of the pre-trained solution with early stopping $c=\hat \evid(\mu_0)$.
However, since $-\hat L(\theta;\mu)$ is not convex with respect to $\mu$, the regularization cannot reproduce the constrained optimization within $\mathcal U$.
To address this issue, we implement the augmented Lagrangian method~\cite{bertsekas2014constrained} for handling the constraint $\mu \in \mathcal U$.

For the weight stability set $\mathcal R$, we employ squared weights simply as a regularization term.
Let 
\begin{align}
    \label{def: weight}
    w_\mu^{(n)} = \frac{a^{(n)}}{\mu(x^{(n)})} + \frac{1-a^{(n)}}{1-\mu(x^{(n)})}
\end{align}
the weight under the nuisance function $\mu$.
Finally, our adversarial objective at the $k$-th epoch is the following.
\begin{align}
    \hat J(\theta, \mu) = &\frac{1}{N}\sum_{n}\left(z^{(n)}_\mu - \tau_\theta(x^{(n)})\right)^2 - \beta \frac{1}{N}\sum_n \left( w^{(n)}_\mu\right)^2 \\
    & - \alpha_k \max\{0, \hat \evid(\mu) -  c\} \\
    & - \lambda_k \left(\max\{0, \hat \evid(\mu) -  c\}\right)^2. \label{eq: adversarial objective}
\end{align}
In each epoch, we minimize $\hat J(\theta, \mu)$ with respect to $\theta$ and maximize with respect to $\mu$.
Overall, our loss controls the error due to the uncertainty of $\mu$ (\ref{eq: generalization error decomposition}-a) by maximizing the first term with respect to $\mu$ under the likelihood constraint in the third and the fourth terms, while simultaneously controlling the estimation variance of the weighted empirical loss (\ref{eq: generalization error decomposition}-b) by flattening the weight with the second term.
The parameters of evidence terms $\alpha_k$ and $\lambda_k$ are updated according to the augmented Lagrangian method. 
The whole algorithm is summarized in \cref{alg1} and \cref{fig: NuNet architecture} illustrates the architecture.

\begin{algorithm}[tb]
    \caption{Nuisance-robust Doubly-robust Network (NuDRNet)}
    \label{alg1}
    \begin{algorithmic}[1]
      \REQUIRE Training data $D = \{(x^{(n)}, a^{(n)}, y^{(n)})\}_n$, hyperparameters $\alpha_0, \gamma, \beta$, validation ratio $r$
      \ENSURE Trained network parameter $\theta$ and validation error
      \STATE Train $f_1,f_0,\mu$ by an arbitrary supervised learning method, e.g.: \\
      $\hat f_a \leftarrow \argmin_{f_a} \frac{1}{N} \sum_{n: a^{(n)}=a} (y^{(n)}-f_a(x^{(n)}))^2$ for each $a\in \{0,1\}$,\\
      $\mu_0 \leftarrow \argmin_{\mu} - \frac{1}{N} \sum_{n} a^{(n)}\log \mu(x^{(n)}) + (1-a^{(n)})\log (1-\mu(x^{(n)}))$
      \STATE Split train and validation $D_\mathrm{tr}, D_\mathrm{val}$ by the ratio $r$ \\
      \STATE Initialize $k \leftarrow 0$, $\mu \leftarrow \mu_0$, and $\theta$ randomly. \\
      \WHILE{Convergence criteria is not met}
      \FOR {Each sub-sampled mini-batch from $D_\mathrm{tr}$}%
      \STATE Update parameters with objective \cref{eq: adversarial objective} and step sizes $\eta_\theta, \eta_\mu$ from optimizes: \\
      \STATE $\theta \leftarrow \theta - \eta_\theta \nabla_\theta \hat J(\theta, \mu)$ \\
      \STATE $\mu \leftarrow \mu + \eta_\mu \nabla_\mu \hat J(\theta, \mu)$ \\
      \ENDFOR
      \STATE $g_k \leftarrow \max\{0, \hat \evid (\mu_k) - \hat \evid (\mu_0)\}$
      \STATE $\alpha_{k+1} \gets \alpha_k + \lambda g_k$
      \IF {Constraint violation is not improved enough, i.e., $g_k < c g_{k-1}$}
        \STATE $\lambda \gets \gamma \lambda $\\
      \ENDIF
      \STATE $k \leftarrow k+1$
      \STATE Check convergence criterion with validation error $\frac{1}{N_\mathrm{val}}\sum_{n \in D_\text{val}} (z_{\mu_0}^{(n)} - \tau_\theta(x^{(n)}))^2$

      \ENDWHILE
      \STATE {\bf return} $\theta$ and the last validation error for model selection
    \end{algorithmic}
  \end{algorithm}

\begin{figure}[htbp]
    \centering
    \subfigure[NuDRNet]{
        \includegraphics[keepaspectratio,width=0.46\textwidth]{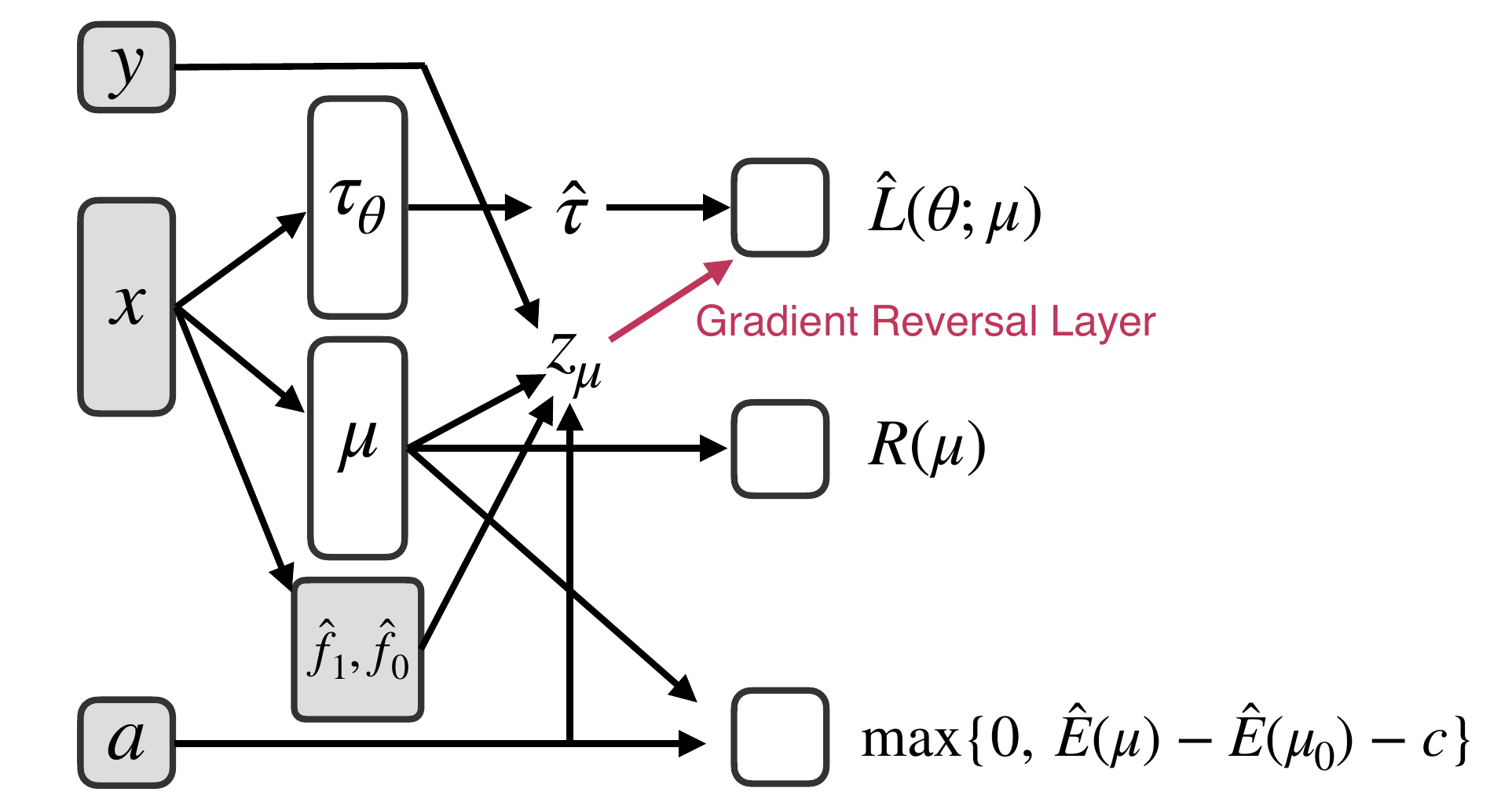}
       \label{fig: nudrnet}
         }
    \subfigure[NuSNet]{
        \includegraphics[keepaspectratio,width=0.46\textwidth]{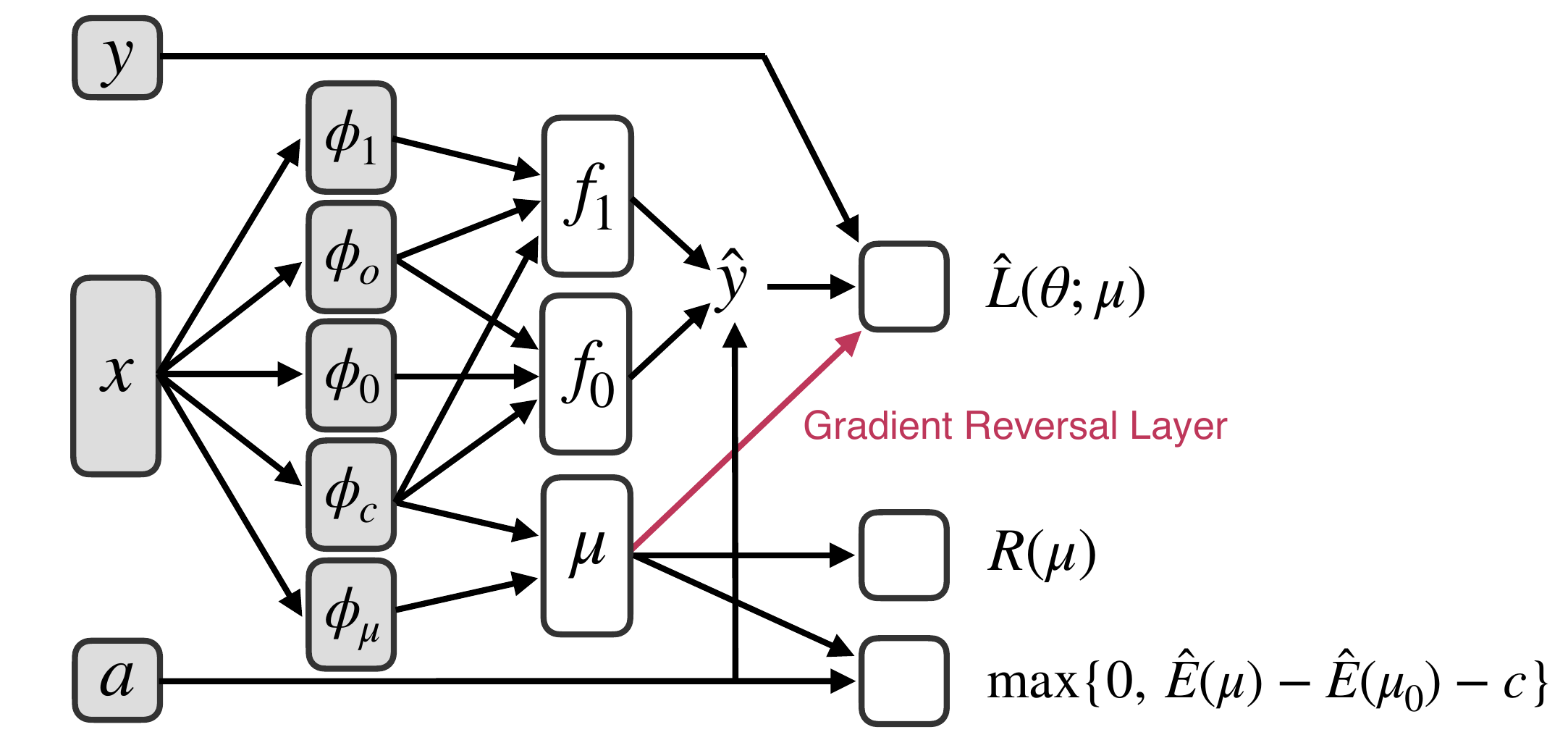}
       \label{fig: nusnet}
         }
    \caption{The training architecture of our network. Gray boxes are pre-trained and fixed. The nuisance function $\mu$ is trained to maximize the empirical loss $\hat L(\theta; \mu)$ while minimizing the other terms. This adversarial formulation can be presented as a joint minimization with the gradient reversal layers indicated in magenta.}
    \label{fig: NuNet architecture}
\end{figure}

\subsection{Nuisance-robust Shared Representation Network (NuSNet)}
\label{subsec: NuSNet}
We further investigate exploiting representation-based architectures.
Deep networks have also performed well in the context of causal inference owing to inductive bias of their representation layers.
Among such methods, SNet~\cite{curth2021nonparametric} has shown outstanding performance thanks to its flexible architecture.
SNet has three heads and five representation extractors.
Specifically, there is a head for $y_1$, $y_0$, and $\mu$, respectively, and a confounder representation $\phi_c$ shared by the three heads, a representation shared by the two outcome heads $\phi_o$, and the representations unique to each hypothesis layer $\phi_1, \phi_0, \phi_\mu$.
Denoting heads as $h_{\cdot}$, the SNet predictor is expressed as follows.
\begin{align}
    f(x, a) = &a h_1(\phi_1(x), \phi_o(x), \phi_c(x)) \\
    &+ (1-a) h_0(\phi_0(x), \phi_o(x), \phi_c(x)) 
\end{align}
The detailed architecture is illustrated in \cref{fig: nudrnet}.

We apply adversarial joint learning to this architecture.
The original SNet does not employ weighted loss.
As noted before, simply applying non-adversarial joint optimization of weight and the outcome model leads to catastrophic results, as seen in our experiment in \cref{sec: appendix additional results}.  %
We therefore first pre-train as in the original SNet and then tune with an adversarially-weighted risk and other regularization terms.
In the tuning phase, the shared representation layer is {\it fixed} and only heads are trained adversarially.

The weighted risk for the tuning phase is
\begin{align}
    \hat L(\theta; \mu) = \frac{1}{\sum_n w_\mu^{(n)}}\sum_n w_\mu^{(n)} (&y^{(n)} - a f_1(x^{(n)};\theta)
    \\ &+ (1-a) f_0(x^{(n)};\theta))^2,
\end{align}
where the instance weight $w_\mu^{(n)}$ is defined as \cref{def: weight}.
The nuisance $\mu$ is trained by {\it maximizing} the above risk, while the target parameter $\theta$ is trained by minimizing it.
The other regularization terms in the overall loss function $\hat J(\theta, \mu)$ are the same as in \cref{eq: adversarial objective}.
The pseudocode of NuSNet is presented in \cref{sec: Algorithm for NuSNet}.

\section{Experiment}
\label{sec: experiment}
To investigate the performance of the proposed methods, we conducted experiments on two synthetic datasets and two real-world datasets.

\subsection{Settings}
\label{subsec: experimental settings}

\paragraph{Synthetic data generation}
We basically followed the setup of~\cite{curth2021nonparametric} for synthetic data generation processes and model hyperparameters.
We generated $d=25$ dimensional multivariate normal covariates including 5 confounding factors that affect outcome and treatment selection, 5 outcome-related factors, and 5 purely CATE-related factors.
The true outcome and CATE models were nonlinear.
Other details are described in \cref{sec: appendix experimental details}.

In addition to the original additive noise (AN) setting $y= y_\mathrm{true} + \varepsilon$, we also tested on the multiplicative noise (MN) setting $y= y_\mathrm{true} (1 + \varepsilon),$ where $\varepsilon$ is drawn from a normal distribution with its average noise level was adjusted to the one in AN.
Most causal inference literature uses synthetic data or semi-synthetic data where only covariates are real and outcomes are synthesized under homogeneous noise, which do not reflect the heterogeneity of the real environment.
Noise heterogeneity is crucial since the optimistic-side error will likely emerge by weighting more on easy (less-noisy) instances.
We therefore set this up as a simple data generation model with heterogeneous noise.

\paragraph*{Real-world dataset}
Most well-established semi-synthetic datasets have real-world covariates and synthesized outcomes and do not reproduce up to noise heterogeneity, while the Twins dataset from~\cite{almond2005costs} had both potential outcomes recorded.
The first-year mortality of twins born at low birth weights was treated as potential outcomes for the heavier and lighter-born twins, respectively, and weighted subsampling was performed to reproduce the missing counterfactuals.
The test target is the difference between the noisy potential outcomes $\tau^{(n)}=y_1^{(n)}-y_0^{(n)}$ instead of CATE $\mathbb E[\tau^{(n)}|x]$.
We followed~\cite{curth2021inductive} for sampling strategy and other details.

Also, the Jobs dataset from~\cite{lalonde1986evaluating} has a randomized test set based on an experimental job training program and an observational training set.
Although we do not have both potential outcomes, we can substitute the true CATE label in the PEHE with the transformed outcome with known true propensity in the randomized test set, as proposed in~\cite{Curth2021ReallyDG}.
A detailed justification for this evaluation is explained in \cref{sec: transformed outcome equivalence}.

We focused on datasets with real-world outcomes to reproduce noise heterogeneity.
Most datasets typically used in causal inference literature have synthetic outcomes with additive noise, i.e., noise homogeneity is assumed.

\paragraph{Baseline methods}
We compared several representative methods for causal inference with DNN.
TNet was a simple plug-in method that estimates each potential outcome with two separate networks and then outputs the difference between their predictions.
SNet and SNet3 were decomposed representation-based methods that shared representation extractors for outcome and propensity estimation layers.
They have three kinds of extractors, namely, outcome-only, propensity-only, and shared representation for confounders.
SNet3 was a modified version, not using weighting in~\cite{curth2021nonparametric}, of what was originally proposed as DR-CFR~\cite{hassanpour2019counterfactual} and DeR-CFR~\cite{wu2022learning} for simultaneous optimization of the weights and the outcome model.
SNet3 with reweighting is presented in \cref{sec: appendix additional results}.
SNet was an even more flexible architecture than SNet3 that had shared and independent extractors for each potential outcome, proposed in~\cite{curth2021nonparametric}.
DRNet was an DNN implementation of DR-Learner~\cite{kennedy2020towards}, in which networks in the first step were independent for each potential outcome and propensity.
In the second step, a single CATE network was trained.

\paragraph{Hyparparameters and model selection}
We set the candidates of the hyperparameters as $\alpha_0 \in \{1, 10\}$, $\gamma \in \{1.5, 2, 3\}$, and $\beta \in \{10, 100, 300\}$.
For the experiment with Twins, we fixed them as $\alpha_0=10, \gamma=1.5, \beta=100.$
Model selection, including early stopping, was performed on the evidence measure of each method with $r=30$\% of the training data held out for validation and the pre-trained weights are used for the proposed methods.

\subsection{Results}

\paragraph{Synthetic datasets}

\begin{table}[tb]
    \centering
    \caption{PEHE on additive noise dataset (mean $\pm$ standard error on 10 runs). The best results are shown in bold, and comparable results are italicized and underlined.
    The five methods in the upper half are non-representation-based, while the five in the lower half are representation-based.
    DRNet-o uses an oracle (true propensity score) and is therefore shaded to indicate that it is unfair as a baseline.
    }
    \label{tb: pehe addnoise}
    \scalebox{0.9}{
    \begin{tabular}{lrrrr}
        \toprule
         Method	& \multicolumn{4}{c}{Additive noise} \\
         {}	&        \multicolumn{1}{c}{N=2000} 	&        \multicolumn{1}{c}{5000} 	&        \multicolumn{1}{c}{10000}	&        \multicolumn{1}{c}{20000} \\
\midrule
    	 TNet 	&   5.02 \scriptsize{$\pm$  0.14}	&   1.96 \scriptsize{$\pm$  0.06}	&   1.22 \scriptsize{$\pm$  0.03}	&   0.88 \scriptsize{$\pm$  0.02} \\
        RNet   &   6.88 \scriptsize{$\pm$ 0.40} &   2.31 \scriptsize{$\pm$ 0.07} &    1.41 \scriptsize{$\pm$ 0.07} &    1.02 \scriptsize{$\pm$ 0.04} \\
        DRNet	&   3.91 \scriptsize{$\pm$  0.14}	&   1.45 \scriptsize{$\pm$  0.04}	&   1.14 \scriptsize{$\pm$  0.11}	&   0.66 \scriptsize{$\pm$  0.03} \\
        \rowcolor{gray!30}
         DRNet-o& 4.15 \scriptsize{$\pm$ 0.16} & 1.37 \scriptsize{$\pm$ 0.04} & 0.82 \scriptsize{$\pm$ 0.02} & 0.53 \scriptsize{$\pm$ 0.01} \\
        \cmidrule{2-5}
        \bf NuDRNet&   4.02 \scriptsize{$\pm$  0.09}	&   1.52 \scriptsize{$\pm$  0.07}	&   0.86 \scriptsize{$\pm$  0.01}	&   0.54 \scriptsize{$\pm$  0.01} \\
        \midrule
         TARNet &   4.28 \scriptsize{$\pm$  0.19} &   1.74 \scriptsize{$\pm$  0.06} &   1.06 \scriptsize{$\pm$  0.02} &   0.76 \scriptsize{$\pm$  0.03} \\
        CFR    &   4.28 \scriptsize{$\pm$  0.19} &   1.71 \scriptsize{$\pm$  0.05} &   1.05 \scriptsize{$\pm$  0.02} &  0.76 \scriptsize{$\pm$  0.03} \\
        SNet3	&   3.85 \scriptsize{$\pm$  0.11}	&   1.54 \scriptsize{$\pm$  0.05}	&   0.99 \scriptsize{$\pm$  0.02}	&   0.62 \scriptsize{$\pm$  0.01} \\
        SNet 	&\bf  3.39 \scriptsize{$\pm$  0.11}	&   1.26 \scriptsize{$\pm$  0.03}	&   0.74 \scriptsize{$\pm$  0.02}	&   0.43 \scriptsize{$\pm$  0.01} \\
        \cmidrule{2-5}
        \bf NuSNet &\underline{\it 3.42 \scriptsize{$\pm$ 0.10}} &\bf 1.14 \scriptsize{$\pm$ 0.03} &\bf 0.60 \scriptsize{$\pm$ 0.01} &\bf 0.32 \scriptsize{$\pm$ 0.01} \\
\bottomrule
    \end{tabular}
    }
\end{table}

\begin{table}[tb]
    \centering
    \caption{PEHE on multiplicative noise dataset (mean $\pm$ standard error on 10 runs). The best results are shown in bold, and comparable results are italicized and underlined.
    The five methods in the upper half are non-representation-based, while the five in the lower half are representation-based.
    DRNet-o uses an oracle (true propensity score) and is therefore shaded to indicate that it is unfair as a baseline.
    }
    \label{tb: pehe mulnoise}
    \scalebox{0.9}{
    \begin{tabular}{lrrrr}
        \toprule
         Method	& \multicolumn{4}{c}{Multiplicative noise} \\
         {}	&        \multicolumn{1}{c}{N=2000} 	&        \multicolumn{1}{c}{5000} 	&        \multicolumn{1}{c}{10000}	&        \multicolumn{1}{c}{20000} \\
\midrule
    	 TNet 	&  11.97 \scriptsize{$\pm$  0.40}	&   5.93 \scriptsize{$\pm$  0.16}	&   3.76 \scriptsize{$\pm$  0.08}	&   2.52 \scriptsize{$\pm$  0.11} \\
        RNet   &  \bf 9.18 \scriptsize{$\pm$ 0.45} &   \underline{\it 5.12 \scriptsize{$\pm$ 0.16}} &    2.98 \scriptsize{$\pm$ 0.09} &    1.91 \scriptsize{$\pm$ 0.06} \\
        DRNet	&  \underline{\it 9.93 \scriptsize{$\pm$  0.40}}  &\underline{\it 4.80 \scriptsize{$\pm$  0.21}}	&   3.20 \scriptsize{$\pm$  0.24}	&   1.83 \scriptsize{$\pm$  0.10} \\
        \rowcolor{gray!30}
         DRNet-o&  9.63 \scriptsize{$\pm$ 0.32} & 4.30 \scriptsize{$\pm$ 0.15} & 2.44 \scriptsize{$\pm$ 0.07} & 1.48 \scriptsize{$\pm$ 0.05} \\
        \cmidrule{2-5}
        \bf NuDRNet&  \underline{\it 9.83 \scriptsize{$\pm$  0.45}}	&\bf  4.67 \scriptsize{$\pm$  0.32}	&\bf  2.44 \scriptsize{$\pm$  0.09}	&\bf  1.50 \scriptsize{$\pm$  0.06} \\
        \midrule
         TARNet &  10.25 \scriptsize{$\pm$  0.36} &   5.27 \scriptsize{$\pm$  0.18} &   3.17 \scriptsize{$\pm$  0.08} &   2.10 \scriptsize{$\pm$  0.09} \\
        CFR    &  10.10 \scriptsize{$\pm$  0.32} &   5.22 \scriptsize{$\pm$  0.18} &   3.16 \scriptsize{$\pm$  0.08} &  2.07 \scriptsize{$\pm$  0.09} \\
        SNet3	&  11.12 \scriptsize{$\pm$  0.36}	&   5.71 \scriptsize{$\pm$  0.25}	&   3.61 \scriptsize{$\pm$  0.14}	&   2.46 \scriptsize{$\pm$  0.09} \\
        SNet 	&  11.22 \scriptsize{$\pm$  0.33}	&   5.47 \scriptsize{$\pm$  0.17}	&   3.12 \scriptsize{$\pm$  0.08}	&   2.01 \scriptsize{$\pm$  0.07} \\
        \cmidrule{2-5}
        \bf NuSNet &11.94 \scriptsize{$\pm$ 0.22} & 5.78 \scriptsize{$\pm$ 0.16} & 3.19 \scriptsize{$\pm$ 0.07} & 1.72 \scriptsize{$\pm$ 0.06} \\
\bottomrule
    \end{tabular}
    }
\end{table}

The results are shown in \cref{tb: pehe addnoise} and \cref{tb: pehe mulnoise}.
Our proposed NuDRNet and NuSNet outperformed or was at least comparable to their baseline methods, DRNet and SNet, respectively.
On the other hand, representation-based methods (SNet3 and SNet) outperformed the transformed outcome methods (DRNet and NuDRNet).
The shared representation extractor of the confounding factors could be an effective inductive bias, especially with small sample sizes.
SNet is overall more accurate than SNet3 since it can also share parameters for components common to both potential outcomes.
Here, the comparison with DRNet-o using true propensity is interesting.
DRNet uses estimated propensity while DRNet-o uses the true one, and DRNet-o showed superior performance in the most cases.
NuDRNet, despite being propensity-agnostic, recovered an accuracy comparable to DRNet-o, which illustrates the benefit of robustness against the ambiguity of nuisance estimation.

\begin{table*}[!tb]
    \centering	\caption{MSE for noisy CATE on real-world datasets (mean $\pm$ standard error over 5 runs).
    It's important to note that for real-world datasets, we can only observe a noisy alternative of the true label $\tau$.
    That is, even as the sample size $N$ approaches infinity, the MSE for this noisy CATE estimate does not converge to zero.
    Only comparisons between methods are meaningful; comparisons to zero are not applicable.
    }
    \label{tb: twins}
    \scalebox{0.9}{
        \begin{tabular}{lllll}
            \toprule
            Method	& \multicolumn{3}{c}{Twins} & \multicolumn{1}{c}{Jobs}\\
            &        \multicolumn{1}{c}{2000} 	&        \multicolumn{1}{c}{5000} 	&        \multicolumn{1}{c}{11400}	& \multicolumn{1}{c}{N=2570}\\
            \midrule
            TNet   &  0.329 {\scriptsize $\pm$ .001} &  0.324 {\scriptsize $\pm$ .002} &  0.322 {\scriptsize $\pm$ .001} & 9.42 {\scriptsize $\pm$ .12}\\
            RNet   &  0.326 {\scriptsize $\pm$ .001} &  0.328 {\scriptsize $\pm$ .002} &  0.322 {\scriptsize $\pm$ .001} & 9.51 {\scriptsize $\pm$ .01}        \\
            DRNet  &\underline{\it 0.322 {\scriptsize $\pm$ .001}} &  0.323 {\scriptsize $\pm$ .001} &  0.323 {\scriptsize $\pm$ .001} & 9.10 {\scriptsize $\pm$ .02 }\\
            \cmidrule{2-5}
            \bf NuDRNet  &\bf 0.320 {\scriptsize $\pm$ .001} &\underline{\it 0.321 {\scriptsize $\pm$ .001}} &\bf 0.319 {\scriptsize $\pm$ .001} &\bf 8.62 {\scriptsize $\pm$ .06}\\
            \midrule
            TARNet &  0.326 {\scriptsize $\pm$ .001} &\underline{\it 0.320 {\scriptsize $\pm$ .001}} &\underline{\it 0.321 {\scriptsize $\pm$ .001}} & 9.33 {\scriptsize $\pm$ .02}\\
            CFR    &\underline{\it 0.323 {\scriptsize $\pm$ .002}} &\underline{\it 0.321 {\scriptsize $\pm$ .001}} &\underline{\it 0.321 {\scriptsize $\pm$ .001}} & 9.33 {\scriptsize $\pm$ .02}\\
            SNet3  &\underline{\it 0.322 {\scriptsize $\pm$ .001}} &\bf 0.319 {\scriptsize $\pm$ .001} &\underline{\it 0.320 {\scriptsize $\pm$ .001}} & 9.38 {\scriptsize $\pm$ .06}\\
            SNet   &\underline{\it 0.323 {\scriptsize $\pm$ .001}} &\underline{\it 0.320 {\scriptsize $\pm$ .001}} &\underline{\it 0.320 {\scriptsize $\pm$ .001}} & 9.36 {\scriptsize $\pm$ .06}\\
            \cmidrule{2-5}
            \bf NuSNet &   0.323 {\scriptsize $\pm$ .003} &\bf 0.319 {\scriptsize $\pm$ .002} &\underline{\it 0.320 {\scriptsize $\pm$ .002}} & 9.51 {\scriptsize $\pm$  .07}\\
            \bottomrule
            \end{tabular}
    }
\end{table*}

\cref{tb: pehe addnoise} also shows the results in the multiplicative noise setting.
NuDRNet outperformed other baselines when the sample size was relatively sufficient.
The pessimistic evaluation with more emphasis on hard instances would be a reasonable explanation for the superiority of the proposed method.
Even though representation decomposition should also be useful in the MN setting since the data generation model was the same as the AN setting except for noise, the weighting approach was superior to the representation decomposition method without weighting.
Again, NuDRNet recovered the DRNet-o with the true propensity.

\paragraph*{Real-world datasets}

Experiments on Twins data also showed the superiority of NuDRNet and NuSNet in most cases as in \cref{tb: twins}.
Note that the test target $\tau^{(n)}$ is noisy, and the value contains the noise variance.
\cref{tb: twins} also showed the results on Jobs data, which exhibits similar trends.
Note that the evaluation metric for Jobs, the MSE with respect to the transformed outcome, also contains constant noise.
Although it may seem that there is not much difference between any of the methods compared to zero, it is possible that there is a significant improvement in PEHE if unobservable constant noise was removed.

\section{Conclusion}

We proposed a learning framework for causal inference with weights in an adversarial end-to-end manner, instead of two-step plug-in estimation, based on our analysis of the generalization error for weighted losses.
Our framework is formulated as distributionally robust optimization (DRO) over a set of nuisance ambiguity with restricted squared weights.
We implemented this framework to doubly robust estimator (DRNet) and the shared representation learner (SNet) as NuDRNet and NuSNet, respectively.
Our proposed methods demonstrated superior performance compared to existing methods not based on weighting and methods based on two-step weighting.
To the best of our knowledge, this approach is the first attempt to apply DRO in causal inference, and it points to a new direction for making various multi-step inference methods end-to-end.
This framework has the potential for wide application to plug-in methods, not limited to the proposed method.

Future challenges include addressing ambiguity in representation layers and stabilizing adversarial learning.
Tackling these challenges is expected to lead to the development of more versatile and practical causal inference methods.

\section*{Acknowledgements}
The Version of Record of this contribution is published in the Neural Information Processing, ICONIP 2025 Proceedings and is available online at \url{https://doi.org/10.1007/978-981-95-4094-5_19}.

\bibliography{reference}

\newpage

\appendix

\onecolumn

\section{Proofs}
\label{sec: appendix proof}
\begin{theorem}[Theorem 4.1]
    Suppose that the instance-wise loss is bounded by $c'$ as $w^{(n)}_\mu \ell^{(n)}(\theta)\le c'$. Then, for any $\delta>0$, with probability at least $1-\delta$ over the choice of a sample $D$, the following holds for all $\theta \in \Theta$.
    \begin{align}
        L(\theta; \mu) - \hat L(\theta; \mu) \le 2 \mathfrak R_\mu(\ell \circ \Theta) + \frac{c'}{2} \sqrt{\frac{\log\left(1/\delta\right)}{N}}.
    \end{align}
\end{theorem}
\begin{proof}
Let $\hat L_D$ denote the empirical loss $\hat L$ with its sample $D$ explicit.
From the McDiarmid's inequality, with probability at least $1-\delta$ we have
\begin{align}
    \sup_{\theta \in \Theta} L(\theta; \mu) - \hat L_D(\theta; \mu) \le 
    & \mathbb E_D \left[\sup_{\theta \in \Theta} L(\theta; \mu) - \hat L_D(\theta; \mu)  \right] + \frac{c'}{2} \sqrt{\frac{\log\left(1/\delta\right)}{N}}. \label{eq: McDiarmid applied}
\end{align}

Let $D'$ be another i.i.d. sample of size $N$ and $(w', \ell')$ be its instance weight and loss, respectively.
The expectation term is bounded by applying the symmetrization as follows.
\begin{align}
    \mathbb E_D &\left[ \sup_{\theta \in \Theta}L(\theta; \mu) - \hat L_D(\theta; \mu)  \right] \\
    =& \mathbb E_D \left[\sup_{\theta \in \Theta} \mathbb E_{D'} \hat L_{D'}(\theta; \mu) - \hat L_D(\theta; \mu)  \right] \\
    \le& \mathbb E_{D, D'}\left[\sup_{\theta \in \Theta} \hat L_{D'}(\theta; \mu) - \hat L_D(\theta; \mu)\right]  \label{eq:symmetrization1} \\ %
    =& \mathbb E_{D,D',\sigma\sim \{\pm 1\}^N}\left[ \sup_{\theta \in \Theta} \frac{1}{N} \sum_{n}^N\sigma_n \left( w'^{(n)}_\mu \ell'^{(n)}(\theta) - w^{(n)}_\mu \ell^{(n)}(\theta) \right)  \right]\\
    \le & \mathbb E_{D',\sigma}\left[ \sup_{\theta \in \Theta} \frac{1}{N} \sum_{n}^N \sigma_n w'^{(n)}_\mu \ell'^{(n)}(\theta) \right] + \mathbb E_{D,\sigma}\left[ \sup_{\theta \in \Theta} \frac{1}{N} \sum_{n}^N \sigma_n  w^{(n)}_\mu \ell^{(n)}(\theta) \right] \label{eq:symmetrization2}\\
    =& 2 \mathbb E_{D,\sigma}\left[ \sup_{\theta \in \Theta} \frac{1}{N} \sum_{n}^N \sigma_n w^{(n)}_\mu \ell^{(n)}(\theta) \right] \\
    \le& 2 \mathbb E_{D,\sigma}\left[ \sup_{\theta \in \Theta} \frac{1}{N} \sum_{n}^N \sigma_n w^{(n)}_\mu f_\theta \left(x^{(n)}\right) \right] \label{eq:symmetrization3} \\
    =& 2 \mathfrak R_\mu(\ell \circ \Theta),
\end{align}
where $\sigma_n \sim \{\pm 1\}$ denotes the Rademacher random variable with equal probability of being 1 or -1, i.e., $P(\sigma_n = 1) = P(\sigma_n = -1) = \frac{1}{2}$.
We applied Jensen's inequality in \cref{eq:symmetrization1} and the element-wise contraction lemma of the Rademacher complexity~\citeapp{tanimoto2022improving} in \cref{eq:symmetrization3}:
\begin{align}
    \mathfrak R(\{\ell_n(a_n)\}) \leq \mathfrak R(\{\rho_n a_n\}),
\end{align}
where $\rho_n$ is the Lipschitz constant of $\ell_n$.

\end{proof}

\begin{theorem}[Theorem 4.2]
    Let $\Theta$ be a bounded linear function class, i.e., $f_\theta(x) = \theta^\top x$ with $\|\theta\|\le B$.
    Furthermore, assume that $\|x\|_2 \le X$ for all $x \in \mathcal X$.
    Then, the following holds:
    \begin{align}
        \mathfrak R_\mu(\ell \circ \Theta) \le \frac{BX}{N} \sqrt{\mathbb E_D \|w_\mu\|_2^2}.
    \end{align}
    \begin{proof}
        We have the following.
        \begin{align}
            \mathfrak R_\mu(\ell \circ \Theta) 
                & = \mathbb E_{D, \sigma} \sup_{\theta} \frac{1}{N} \sum_n \sigma_n w_{\mu}^{(n)} \theta^{\top} x^{(n)} \\
                & = \mathbb E_{D, \sigma} \sup _\theta \theta^\top \frac{1}{N} \sum_n \sigma_n w_{\mu}^{(n)} x^{(n)} \\
                & = \mathbb{E}_{} \frac{B}{N}\left\|\sum \sigma_n w^{(n)} x^{(n)}\right\|_2 \\
                & \le \frac{B}{N} \sqrt{\mathbb E_{D, \sigma}\left\|\sum_n \sigma_n w^{(n)} x^{(n)}\right\|_2^2} \\
                & \le \frac{BX}{N} \sqrt{\mathbb E_D \|w_\mu\|_2^2}.
        \end{align}
        We applied Jensen's inequality in the penultimate step and the Cauchy-Schwarz inequality in the final step.
    \end{proof}
\end{theorem}

\section{Equivalence between the transformed outcome and the CATE in expectation}
\label{sec: transformed outcome equivalence}
Our empirical risk and the evaluation criteria are based on the method of transformed outcome~\citeapp{athey2016recursive,horvitz1952generalization}.
The inverse-probability-weighted (IPW) transformed outcome is
\begin{align}
    \label{eq: transformed outcome PW}
    z^{(n)} = \frac{a^{(n)} y^{(n)}}{\mu(x^{(n)})} - \frac{\left(1-a^{(n)}\right) y^{(n)}}{1-\mu(x^{(n)})}.
\end{align}
The IPW transformed outcome $z$ in \cref{eq: transformed outcome PW} is equivalent to the CATE in the sense of its conditional expectation:
\begin{align}
    \mathbb E[z|x] &= \mathbb E_{Y_0,Y_1,A \sim \mu(x)} \left[Y_1 \frac{A}{\mu(x)} - Y_0 \frac{1-A}{1-\mu(x)} \middle| x \right] \\
&= \mathbb E[Y_1-Y_0 | x] =: \tau(x).
\end{align}
Then, letting $z=\tau(x) + \varepsilon$ with $\mathbb E[\varepsilon|x]=0$, we have \begin{align}
    &\mathbb E_{z,x}[(z-\hat \tau)^2] \\
    &= \mathbb E_{\varepsilon,x}[(\tau(x)+\varepsilon-\hat \tau)^2] \\
    &= \mathbb E_x[(\tau(x)-\hat \tau)^2] + 2\mathbb E_x [ \mathbb E_\varepsilon[\varepsilon|x] (\tau(x)-\hat \tau)] + \mathbb E_{\varepsilon,x}[\varepsilon^2] \\
    &= \mathrm{PEHE} +  \mathbb E_x \mathbb V[\varepsilon].
\end{align}
The MSE on $z$ is equivalent to our final metric PEHE except for a constant term.
The same equivalence can be derived for the doubly robust transformed outcome in \cref{eq: transformed outcome DR}.
This equivalence justifies our employed MSE on $z$ as the empirical risk $\hat L$ and the evaluation metric.

\section{Algorithm for NuSNet}
\label{sec: Algorithm for NuSNet}
\cref{alg2} presents the full algorithm of our proposed method in \cref{subsec: NuSNet}.
\begin{algorithm}[tb]
    \caption{Nuisance-robust Shared Representation Network (NuSNet)}
    \label{alg2}
    \begin{algorithmic}[1]
    \REQUIRE Training data $D = \{(x^{(n)}, a^{(n)}, y^{(n)})\}_n$, hyperparameters $\rho, \epsilon$, validation ratio $r$
    \ENSURE Trained network parameters $\theta_0, \theta_1, \phi$ and validation error
    \STATE Train $f_0,f_1,\pi$ by an arbitrary supervised learning method, e.g.: \\
    $\hat f_a \leftarrow \argmin_{f_a} \frac{1}{N} \sum_{n: a^{(n)}=a} (y^{(n)}-f_a(x^{(n)}))^2$ for each $a\in \{0,1\}$,\\
    $\pi_0 \leftarrow \argmin_{\pi} - \frac{1}{N} \sum_{n} a^{(n)}\log \pi(x^{(n)}) + (1-a^{(n)})\log (1-\pi(x^{(n)}))$
    \STATE Split train and validation $D_\mathrm{tr}, D_\mathrm{val}$ by the ratio $r$ \\
    \STATE $k \leftarrow 0$, $\alpha \leftarrow 0$, $\lambda \leftarrow 1$ \\
    \WHILE{Convergence criteria is not met}
    \FOR {each sub-sampled mini-batch from $D_\mathrm{tr}$}
    \STATE Update parameters with objective \cref{eq: adversarial objective} and step sizes $\eta_{\theta_0}, \eta_{\theta_1}, \eta_\phi$ from optimizes: \\
    \STATE $\theta_{0} \leftarrow \theta_{0} - \eta_{\theta_0} \nabla_{\theta_0} \hat J(\theta_0, \theta_1, \phi; \alpha, \lambda)$ \\
    \STATE $\theta_{1} \leftarrow \theta_{1} - \eta_{\theta_1} \nabla_{\theta_1} \hat J(\theta_0, \theta_1, \phi; \alpha, \lambda)$ \\
    \STATE $\phi \leftarrow \phi + \eta_\phi \nabla_\phi \hat J(\theta_0, \theta_1, \phi; \alpha, \lambda)$ \\
    \ENDFOR
    \STATE Check convergence criterion with validation error $\frac{1}{N_\mathrm{val}}\sum_{n \in D_\text{val}} (y^{(n)} - \hat{y}^{(n)})^2$ where $\hat{y}^{(n)} = a^{(n)} f_1(x^{(n)};\theta_1) + (1-a^{(n)}) f_0(x^{(n)};\theta_0)$
    \STATE $g_k \leftarrow \max\{0, \hat \evid (\mu_k) - \hat \evid (\mu_0)\}$
    \STATE $\alpha_{k+1} \gets \alpha_k + \lambda g_k$
    \IF {Constraint violation is not improved enough, i.e., $g_k < c g_{k-1}$}
      \STATE $\lambda \gets \gamma \lambda $\\
    \ENDIF
    \STATE $k \leftarrow k+1$
    \ENDWHILE
    \STATE {\bf return} $\theta_0, \theta_1, \phi$ and the last validation error for model selection
    \end{algorithmic}
\end{algorithm}

\section{Experimental details}
\label{sec: appendix experimental details}

\subsection{Simulation environment}
Our synthetic data in the additive noise (AN) setting was identical to the setting used in~\citeapp{curth2021nonparametric}, which was inspired by the decomposed covariate setting used in~\citeapp{Hassanpour2020Learning}.
We used $d=25$ dimensional normal covariates $x$.
Out of $25$ covariates, there were $d_o=5$ outcome-related covariates $x_o$ that affect only potential outcomes, $d_c=5$ confounders $x_c$ that affect both potential outcomes and treatment assignment, and $d_t=5$ covariates that affect treatment effect $x_t$.

The expected potential outcomes were calculated as follows.
\begin{align}
    \mathbb E[Y_0|x] &= \mathbf 1^\top \begin{bmatrix}
        x_c\\
        x_o
    \end{bmatrix}^2, \\
    \mathbb E[Y_1|x] &= \mathbb E[Y_0|x] + \tau(x),
\end{align}
Squaring works on an element-by-element basis
where $\mathbf 1 = [1,\cdots,1]^\top,$ squaring ${\cdot}^2$ is element-wise, and the treatment effect $\tau(x)$ was defined as
\begin{align}
    \tau(x) = \mathbf 1^\top x_t^2.
\end{align}
The true propensity that affects the treatment assignment was defined as
\begin{align}
    \mu(x)=\mu(a=1|x) = \sigma(\xi (\mathbf 1^\top x_c^2/d_c - \omega)),
\end{align}
where $\sigma$ was the sigmoid, $\xi$ was the strength of selection and was set as $\xi = 3$, and $\omega$ was adaptively set so that the median of the inside $\sigma$ would be $0$.

The expected factual outcome is defined as follows.
\begin{align}
    \bar y = A \mathbb E[Y_1|x] + (1-A) \mathbb E[Y_0|x]
\end{align}
where $A \sim \mathrm{Bernoulli}(\mu(A=1|x))$.
In the AN setting, the outcome was observed with additive noise as
\begin{align}
    y= \bar y + \varepsilon,
\end{align}
where $\varepsilon \sim \mathcal{N}(0,1)$.
In the multiplicative noise (MN) setting, it was
\begin{align}
    y= \bar y (1+ \varepsilon'),
\end{align}
where $\varepsilon' \sim \mathcal N(0,\xi)$ with its standard deviation $\xi=2 \big/ \left(\sqrt{\mathrm{Var}[\mathbb E[Y_1|x]]} + \sqrt{\mathrm{Var}[\mathbb E[Y_0|x]]}\right)$.

\subsection{Architecture and hyperparameters}
\paragraph{Synthetic data experiment}
We followed the implementation of~\citeapp{curth2021nonparametric} (BSD 3-Clause License) for the synthetic data experiment.
We employed the multi-layer perceptron with representation (input-side) layers and hypothesis (output-side) layers.
For the separated models (TNet, PWNet, DRNet, and NuDRNet), the representation layers were 3 layers with 200 units each and the hypothesis layers were 2 layers with 100 units each, for each prediction of $y_0$, $y_1$, and $a$.
SNet3 had 3 representations of outcome-only (50 units $\times$ 3 layers), treatment-only (50 units $\times$ 3 layers), and shared representation layers (150 units $\times$ 3 layers).
SNet had 5 representations for outcome related 3 representations ($y_1$-only, $y_0$-only, and outcome-shared, of 50 units $\times$ 3 layers), treatment-only (100 units $\times$ 3 layers), and shared for outcomes and treatment (100 units $\times$ 3 layers).
We used the exponential linear unit (ELU) activations~\citeapp{Clevert2016fast} and the optimizer was Adam~\citeapp{kingma2015adam}.
For NuDRNet, we applied model selection for the training epochs with the same validation data due to the instability of adversarial training.
NuDRNet sometimes diverges within the minimum training epochs or the early-stopping patience epochs of 10.
Therefore, it would be better to keep the best-so-far parameters in each epoch and output it.

\paragraph*{Real-world dataset experiment}
We followed~\citeapp{curth2021inductive} for the Twins experiment.
The representation layers and the hypothesis layers were only 1 in all methods.
For NuDRNet, instead of model selection of epochs, shorter minimum epochs of 40 were used to avoid overfitting the pre-trained $\mu$, as opposed to 200 in other methods as in~\citeapp{curth2021inductive}.

\paragraph*{Infrastructure} All the experiments were run on a
machine with 28 CPUs (Intel(R) Xeon(R) CPU E5-
2680 v4 @ 2.40GHz), 250GB memory, and 8 GPUs (NVIDIA GeForce GTX 1080 Ti).

\section{Additional results}
\label{sec: appendix additional results}
We present additional experimental results of the experiment on the synthetic dataset.
We additionally tested the following methods proposed in the original literature.
\begin{itemize}
    \item DR-Learner and R-Learner with cross-fitting as originally proposed~\citeapp{kennedy2020towards,nie2021quasi}.
    \item SNet3 with reweighted loss (DeR-CFR)~\citeapp{wu2022learning}
\end{itemize}

The results are shown in \cref{tb: pehe addnoise appendix} and \cref{tb: pehe mulnoise appendix}.
In our tested settings, these methods did not perform well.
Despite the asymptotic guarantees of these methods, they tend to set extreme weights, making it difficult to achieve consistent performance with realistic sample sizes, especially when models are complex or the overlap is limited.

\begin{table*}[tb]
    \centering
    \caption{PEHE on additive noise dataset (mean $\pm$ standard error on 10 runs). The best results are shown in bold, and comparable results are italicized and underlined.
    }
    \label{tb: pehe addnoise appendix}
    \scalebox{\tablescale}{
    \begin{tabular}{lrrrrrr}
        \toprule
         Method	&    \multicolumn{1}{c}{N=500}  	&        \multicolumn{1}{c}{1000} 	&        \multicolumn{1}{c}{2000} 	&        \multicolumn{1}{c}{5000} 	&        \multicolumn{1}{c}{10000}	&        \multicolumn{1}{c}{20000} \\
\midrule
    	 TNet 	&  18.55 \scriptsize{$\pm$  0.88}	&  13.89 \scriptsize{$\pm$  1.10}	&   5.02 \scriptsize{$\pm$  0.14}	&   1.96 \scriptsize{$\pm$  0.06}	&   1.22 \scriptsize{$\pm$  0.03}	&   0.88 \scriptsize{$\pm$  0.02} \\
         TARNet &  18.21 \scriptsize{$\pm$  1.12} &   8.77 \scriptsize{$\pm$  0.35} &   4.28 \scriptsize{$\pm$  0.19} &   1.74 \scriptsize{$\pm$  0.06} &   1.06 \scriptsize{$\pm$  0.02} &   0.76 \scriptsize{$\pm$  0.03} \\
         CFR    &  17.90 \scriptsize{$\pm$  1.18} &   8.77 \scriptsize{$\pm$  0.35} &   4.28 \scriptsize{$\pm$  0.19} &   1.71 \scriptsize{$\pm$  0.05} &   1.05 \scriptsize{$\pm$  0.02} &  0.76 \scriptsize{$\pm$  0.03} \\
         SNet3	&\bf 13.10 \scriptsize{$\pm$  0.65}	& \underline{\it 7.73 \scriptsize{$\pm$  0.34}}	&   3.85 \scriptsize{$\pm$  0.11}	&   1.54 \scriptsize{$\pm$  0.05}	&   0.99 \scriptsize{$\pm$  0.02}	&   0.62 \scriptsize{$\pm$  0.01} \\
         SNet3 w/ reweighting &	61.47 \scriptsize{$\pm$ 2.78} &	63.13 \scriptsize{$\pm$ 2.32} &	63.79 \scriptsize{$\pm$ 2.27} &	69.28 \scriptsize{$\pm$ 1.30} &	72.20 \scriptsize{$\pm$ 1.82} &	74.67 \scriptsize{$\pm$ 1.89} \\
         SNet 	&\underline{\it 14.14 \scriptsize{$\pm$  0.57}}	&\bf 7.17 \scriptsize{$\pm$  0.29}	&\bf  3.39 \scriptsize{$\pm$  0.11}	&\bf  1.26 \scriptsize{$\pm$  0.03}	&\bf  0.74 \scriptsize{$\pm$  0.02}	&\bf  0.43 \scriptsize{$\pm$  0.01} \\
         RNet &   18.47 \scriptsize{$\pm$ 3.25} &   13.65 \scriptsize{$\pm$ 1.14} &   6.88 \scriptsize{$\pm$ 0.40} &   2.31 \scriptsize{$\pm$ 0.07} &    1.41 \scriptsize{$\pm$ 0.07} &    1.02 \scriptsize{$\pm$ 0.04} \\
         RNet w/ cross-fit &  15.28 \scriptsize{$\pm$ 0.90} &   11.79 \scriptsize{$\pm$ 0.50} &   6.46 \scriptsize{$\pm$ 0.27} &   2.63 \scriptsize{$\pm$ 0.08} &    1.47 \scriptsize{$\pm$ 0.04} &    0.92 \scriptsize{$\pm$ 0.03} \\   
         PWNet	&  18.46 \scriptsize{$\pm$  0.82}	&  13.03 \scriptsize{$\pm$  0.54}	&  15.97 \scriptsize{$\pm$  0.68}	&  20.99 \scriptsize{$\pm$  1.25}	&  25.31 \scriptsize{$\pm$  2.32}	&  19.21 \scriptsize{$\pm$  1.36} \\
         DRNet	&  16.56 \scriptsize{$\pm$  0.75}	&  11.58 \scriptsize{$\pm$  0.66}	&   3.91 \scriptsize{$\pm$  0.14}	&   1.45 \scriptsize{$\pm$  0.04}	&   1.14 \scriptsize{$\pm$  0.11}	&   0.66 \scriptsize{$\pm$  0.03} \\
         DRNet w/ cross-fit &	162.27 \scriptsize{$\pm$ 75.79} &	18.83 \scriptsize{$\pm$ 1.55} &	9.07 \scriptsize{$\pm$ 0.53} &	6.51 \scriptsize{$\pm$ 1.38} &	5.90 \scriptsize{$\pm$ 1.19} &	3.94 \scriptsize{$\pm$ 1.41} \\
         \midrule
         NuDRNet	&  15.78 \scriptsize{$\pm$  0.69}	&  11.43 \scriptsize{$\pm$  0.48}	&   4.02 \scriptsize{$\pm$  0.09}	&   1.52 \scriptsize{$\pm$  0.07}	&   0.86 \scriptsize{$\pm$  0.01}	&   0.54 \scriptsize{$\pm$  0.01} \\

\bottomrule
    \end{tabular}
    }
\end{table*}

\begin{table*}[tb]
    \caption{PEHE on multiplicative noise dataset}
    \label{tb: pehe mulnoise appendix}
    \centering
    \scalebox{\tablescale}{
    \begin{tabular}{lrrrrrr}
        \toprule
         Method	&    \multicolumn{1}{c}{N=500}  	&        \multicolumn{1}{c}{1000} 	&        \multicolumn{1}{c}{2000} 	&        \multicolumn{1}{c}{5000} 	&        \multicolumn{1}{c}{10000}	&        \multicolumn{1}{c}{20000} \\
\midrule

    	 TNet 	&  22.03 \scriptsize{$\pm$  1.23}	&  17.59 \scriptsize{$\pm$  0.89}	&  11.97 \scriptsize{$\pm$  0.40}	&   5.93 \scriptsize{$\pm$  0.16}	&   3.76 \scriptsize{$\pm$  0.08}	&   2.52 \scriptsize{$\pm$  0.11} \\
         TARNet &  20.75 \scriptsize{$\pm$  0.99} &\bf  13.06 \scriptsize{$\pm$  0.53} & 10.25 \scriptsize{$\pm$  0.36} &   5.27 \scriptsize{$\pm$  0.18} &   3.17 \scriptsize{$\pm$  0.08} &   2.10 \scriptsize{$\pm$  0.09} \\
         CFR    &  20.30 \scriptsize{$\pm$  1.05} &\bf  13.06 \scriptsize{$\pm$  0.53} & 10.10 \scriptsize{$\pm$  0.32} &   5.22 \scriptsize{$\pm$  0.18} &   3.16 \scriptsize{$\pm$  0.08} &  2.07 \scriptsize{$\pm$  0.09} \\
         SNet3	&\bf 17.83 \scriptsize{$\pm$  0.94}	&  15.44 \scriptsize{$\pm$  0.65}	&  11.12 \scriptsize{$\pm$  0.36}	&   5.71 \scriptsize{$\pm$  0.25}	&   3.61 \scriptsize{$\pm$  0.14}	&   2.46 \scriptsize{$\pm$  0.09} \\
         SNet3 w/ reweighting &	60.87 \scriptsize{$\pm$ 2.58} &	61.92 \scriptsize{$\pm$ 2.04} &	64.04 \scriptsize{$\pm$ 1.93} &	67.63 \scriptsize{$\pm$ 0.80} &	70.63 \scriptsize{$\pm$ 1.73} &	74.21 \scriptsize{$\pm$ 1.60} \\
         SNet 	&\underline{\it 18.44 \scriptsize{$\pm$  0.86}}	&  15.73 \scriptsize{$\pm$  0.58}	&  11.22 \scriptsize{$\pm$  0.33}	&   5.47 \scriptsize{$\pm$  0.17}	&   3.12 \scriptsize{$\pm$  0.08}	&   2.01 \scriptsize{$\pm$  0.07} \\
         RNet &   32.26 \scriptsize{$\pm$ 6.86} &   11.87 \scriptsize{$\pm$ 0.30} &   \bf 9.18 \scriptsize{$\pm$ 0.45} &   \underline{\it 5.12 \scriptsize{$\pm$ 0.16}} &    2.98 \scriptsize{$\pm$ 0.09} &    1.91 \scriptsize{$\pm$ 0.06} \\
         RNet w/ cross-fit &   18.74 \scriptsize{$\pm$ 1.34} &   13.98 \scriptsize{$\pm$ 0.62} &  10.64 \scriptsize{$\pm$ 0.49} &   5.81 \scriptsize{$\pm$ 0.16} &    3.45 \scriptsize{$\pm$ 0.08} &    1.99 \scriptsize{$\pm$ 0.06} \\
         PWNet	&\underline{\it 18.97 \scriptsize{$\pm$  0.90}}	&\underline{\it 13.14 \scriptsize{$\pm$  0.54}}	&  15.95 \scriptsize{$\pm$  0.64}	&  21.08 \scriptsize{$\pm$  1.22}	&  25.63 \scriptsize{$\pm$  2.31}	&  20.92 \scriptsize{$\pm$  2.09} \\
         DRNet	&  19.96 \scriptsize{$\pm$  1.01}	&  15.34 \scriptsize{$\pm$  0.75}	&\underline{\it 9.93 \scriptsize{$\pm$  0.40}}  &\underline{\it 4.80 \scriptsize{$\pm$  0.21}}	&   3.20 \scriptsize{$\pm$  0.24}	&   1.83 \scriptsize{$\pm$  0.10} \\
         DRNet w/ cross-fit &	179.60 \scriptsize{$\pm$ 116.17} &	21.65 \scriptsize{$\pm$ 1.95} &	12.80 \scriptsize{$\pm$ 0.70} &	10.42 \scriptsize{$\pm$ 0.75} &	9.68 \scriptsize{$\pm$ 1.49} &	7.14 \scriptsize{$\pm$ 1.32} \\
         \midrule
         NuDRNet	&\underline{\it 19.96 \scriptsize{$\pm$  1.20}}	&  15.54 \scriptsize{$\pm$  0.57}	& \underline{\it 9.83 \scriptsize{$\pm$  0.45}}	&\bf  4.67 \scriptsize{$\pm$  0.32}	&\bf  2.44 \scriptsize{$\pm$  0.09}	&\bf  1.50 \scriptsize{$\pm$  0.06} \\
    \bottomrule
    \end{tabular}
    }
\end{table*}

\end{document}